\documentclass{article} 
\usepackage{iclr2027_conference,times}

\usepackage{amsmath,amsfonts,bm}

\def\eqref#1{equation~\ref{#1}}

\def\1{\bm{1}}

\def\vk{{\bm{k}}}

\def\vq{{\bm{q}}}

\def\vv{{\bm{v}}}

\DeclareMathAlphabet{\mathsfit}{\encodingdefault}{\sfdefault}{m}{sl}
\SetMathAlphabet{\mathsfit}{bold}{\encodingdefault}{\sfdefault}{bx}{n}

\usepackage{hyperref}
\usepackage{url}
\usepackage[utf8]{inputenc}
\usepackage[T1]{fontenc}
\usepackage{booktabs}
\usepackage{amsfonts}
\usepackage{amsmath}
\usepackage{amssymb}
\usepackage{amsthm}
\usepackage{nicefrac}
\usepackage{microtype}
\usepackage[table]{xcolor}
\usepackage{graphicx}
\usepackage{wrapfig}
\usepackage{algorithm}
\usepackage{algpseudocode}
\usepackage{multirow}
\usepackage{makecell}
\usepackage{longtable}

\newtheorem{theorem}{Theorem}
\newtheorem{corollary}{Corollary}
\theoremstyle{definition}
\newtheorem{definition}{Definition}
\theoremstyle{remark}
\newtheorem{remark}{Remark}

\newcommand{\Rot}[1]{\mathbf{R}(#1)}
\renewcommand{\vq}{\mathbf{q}}
\renewcommand{\vk}{\mathbf{k}}
\renewcommand{\vv}{\mathbf{v}}
\DeclareMathOperator{\geomspace}{geomspace}
\definecolor{relgreen}{RGB}{0,128,0}
\newcommand{\rel}[1]{{\fontsize{6}{7}\selectfont\textcolor{relgreen}{\textbf{+#1\%}}}}

\title{T-RoPE: Time-Aware Rotary Position Embedding for Sequential Recommendation}

\author{
\normalfont\textbf{Yang Liu\textsuperscript{1} \quad
Noel Loo\textsuperscript{2,3} \quad
Ali Khanafer\textsuperscript{1} \quad
Shuying Sun\textsuperscript{1}\textsuperscript{\textdagger} \quad
Akshay Soni\textsuperscript{1}\textsuperscript{\textdagger}} \\
\textbf{Zhong Wu\textsuperscript{1} \quad
Linjun Yang\textsuperscript{1}} \\
\textsuperscript{1}Shopify \qquad
\textsuperscript{2}Massachusetts Institute of Technology \qquad
\textsuperscript{3}Liquid AI
}

\iclrfinalcopy 
\begin{document}

\maketitle
\lhead{Preprint}
\begingroup
\renewcommand{\thefootnote}{}
\footnotetext{\textsuperscript{\textdagger}\textbf{Corresponding authors:}
\texttt{shuying.sun@shopify.com} and
\texttt{akshay.soni@shopify.com}.}
\endgroup

\begin{abstract}
Large-scale recommenders increasingly adopt the sequential generative recipe behind large language models, bringing the Transformer into recommendation along with design choices made for text, including Rotary Position Embedding (RoPE). In language models, RoPE encodes token indices for relative position reasoning, but in recommendation, an interaction index records only event order, saying nothing about elapsed time, behavioral cycles across scales, or calendar phase. We revisit this choice and propose T-RoPE, a time-aware RoPE for sequential generative recommendation that replaces index-only rotation with timestamp-based angles, learnable temporal coefficients, multiscale frequency banks, shifted query alignment, and non-stationary key rotation. We prove that standard RoPE, even on timestamps, remains time-translation invariant and cannot distinguish seasonal contexts, and that T-RoPE breaks this invariance while preserving the RoPE interface. Across five public benchmarks, T-RoPE achieves the best result on every metric on every dataset, improving over the strongest baseline by 78--130\% in HR@10 on the sparse PixelRec data and 8--12\% across metrics on Amazon Books. On an industrial-scale e-commerce dataset with more than 6B interactions, it improves every metric over the HSTU + Time RAB backbone by 13--82\%, with ablations attributing the largest gains to multiscale frequencies ($+56\%$ NDCG@50) and non-stationary keys ($+4\%$). An online A/B test in the Shop app yields positive lifts in conversion rate ($+0.33\%$) and order count ($+0.63\%$). We also provide forward and backward algorithms whose added cost is linear in sequence length and head dimension, keeping time-aware RoPE practical for large generative recommenders.
\end{abstract}

\section{Introduction}
\label{sec:intro}

Large language models have made autoregressive sequence modeling a practical way to
learn from long, heterogeneous histories~\citep{vaswani2017attention,touvron2023llama}.
Their success comes not from the architecture alone but from a mature engineering stack of scalable training, fast attention kernels, and long-context modeling that makes generative prediction practical at scale.
Recommendation models are now borrowing this recipe. Recent sequential generative
recommenders represent user behavior as action sequences and autoregressively predict future items,
actions, or semantic identifiers
~\citep{rajput2023tiger,wang2024letter,chen2024hllm,hou2025actionpiece,onereason2026,liang2026gr2}.
This recipe has already proven itself at production scale, where HSTU drives live recommendation systems on industrial workloads~\citep{zhai2024hstu}.

This transfer is productive but imperfect, because inductive biases that serve language do not always carry over to recommendation. Generative recommenders inherit
components from language models, including RoPE~\citep{su2021roformer,tian2026mrrope,liu2025ropepp,movahedi2026selective}. RoPE injects sequence order by rotating query and key
vectors according to an integer position index. For text, this is a natural coordinate, because
token offsets carry much of the structure attention needs.  For recommendation, the same coordinate loses what matters. An interaction index records that one event preceded another, but not whether the gap was seconds, weeks, or seasons.

This positional encoding gap leaves three temporal signals unrepresented. First, \emph{true time
gaps} are erased. A key six months in the past can be positionally identical to a key
from yesterday if their sequence offsets match. Second, \emph{multiscale periodicity}
is absent. User behavior contains hourly, weekly, monthly, and seasonal rhythms, but
index-based RoPE has no wall-clock frequency structure. Third, \emph{absolute calendar phase}
is inaccessible. Because index-based attention depends only on the relative offset
$m - n$, it is time-translation invariant. The same history evaluated at different
points in a seasonal cycle produces the same attention pattern.

\begin{table}[t]
\vspace{-10pt}
\caption{Temporal mechanisms explicitly provided by positional encoding methods.} 
\label{tab:comparison}
\centering
\small
\begin{tabular}{lccccc}
\toprule
Method & Timestamp & Learnable & Multiscale & Shifted & Non-stationary \\
\midrule
Standard RoPE~\citep{su2021roformer} & \texttimes & \texttimes & \texttimes & \texttimes & \texttimes \\
TiSASRec~\citep{li2020time}          & Partial    & \texttimes & \texttimes & \texttimes & \texttimes \\
TO-RoPE~\citep{wei2025torope}        & \checkmark & Partial    & \checkmark & \texttimes & \texttimes \\
\textbf{T-RoPE (ours)}               & \checkmark & \checkmark & \checkmark & \checkmark & \checkmark \\
\bottomrule
\end{tabular}
\par\begin{minipage}{0.98\linewidth}
\scriptsize
Multiscale denotes an explicit continuous frequency bank covering multiple periods.
Learnable denotes learned temporal coefficients inside the encoding (not generic model
parameters). Partial denotes a limited or indirect version of the property.
\end{minipage}
\end{table}


Existing temporal recommendation methods close only part of this gap, as Table~\ref{tab:comparison} summarizes. TiSASRec injects
discretized time-interval features into self-attention, which captures recency but not
continuous multiscale periodicity or calendar phase~\citep{li2020time}. 
HSTU adds a time-based Relative Attention Bias (RAB), mapping bucketed time gaps to a learned scalar added to
the attention logits, which sharpens recency weighting but leaves the query and key
geometry untouched and still discretizes time~\citep{zhai2024hstu}.
TO-RoPE brings timestamps into RoPE angles, but the joint geometry of its queries and keys remains
time-translation invariant~\citep{wei2025torope}.  As we prove in Section~\ref{sec:theory}, it captures temporal distance but cannot distinguish queries issued in different calendar phases over the same history. Time cannot just sit alongside RoPE as an extra feature or bias, because closing the gap requires structurally changing the way RoPE itself rotates queries and keys.

To address this gap, we propose \textbf{T-RoPE} (\textbf{T}ime-aware \textbf{RoPE}), a RoPE-compatible
positional encoding for timestamped user histories. T-RoPE keeps the efficient RoPE
interface but replaces index-only rotation with timestamp-based angles, learnable
temporal coefficients, multiscale frequency banks, shifted query alignment, and
non-stationary key rotation, as shown in Table~\ref{tab:comparison}. We list our main contributions below.
\begin{enumerate}
    \item \textbf{A time-aware RoPE for generative recommendation.} T-RoPE encodes
    continuous timestamps directly in rotary attention to represent elapsed time,
    multiscale periodicity, and calendar phase within a single rotary encoding. Index-based
    RoPE discards these signals, and prior temporal encodings capture them only in part.

    \item \textbf{Theoretical analysis.} We prove that rotating queries and keys by the same
    function of their timestamps yields time-translation invariant attention (covering
    standard RoPE on timestamps and TO-RoPE's timestamp component), and that T-RoPE's
    non-stationary key provably breaks this invariance.

    \item \textbf{Results across settings.} On five public benchmarks, an industrial dataset
    with more than 6B interactions, and an online A/B test in the Shop app, T-RoPE improves
    over temporal and RoPE-based baselines on every public benchmark and delivers positive
    production lifts.

    \item \textbf{An efficient implementation path.} We provide forward and backward algorithms
    and a cost analysis showing only linear overhead in sequence length and head dimension. 
\end{enumerate}

\section{Preliminaries}
\label{sec:prelim}

In generative recommendation, the model input is a user's interaction history, a sequence of timestamped events $\{(i_1, t_1), \ldots, (i_L, t_L)\}$ ordered by time $t_1 < t_2 < \cdots < t_L$. Each event contains an item $i_m$ and a Unix timestamp $t_m \in \mathbb{R}_{>0}$ measured in seconds. It may also include an interaction type or other semantic identifiers. The model predicts the next event autoregressively, and we take the prediction target to be the next item $i_{L+1}$.

\begin{figure}[t]
    \centering
    \vspace{-10pt}
    \includegraphics[width=0.98\linewidth,trim={1 7 5 3},clip]{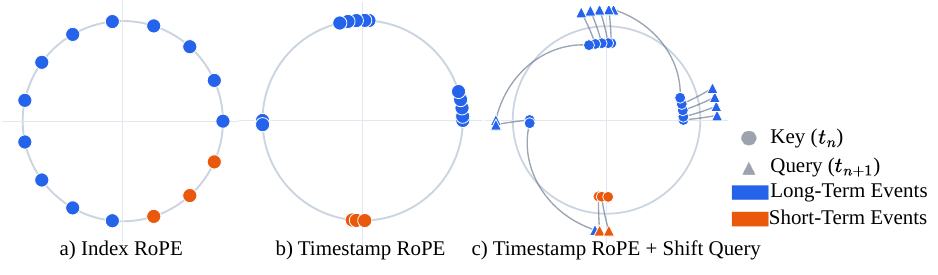}
    \vspace{-6pt}
    \caption{Geometric intuition for T-RoPE rotations. a) Index RoPE places events at uniform sequence positions, hiding elapsed time. b) Timestamp RoPE uses event time, so bursty events cluster and dormant gaps spread on the rotation circle. c) Shifted queries place each query at the next-event timestamp (markers displaced from keys), matching training to the inference-time geometry.}
    \label{fig:overview}
\end{figure}

A Transformer encoder maps each item $i_m$ to query, key, and value vectors $\vq_m, \vk_m, \vv_m \in \mathbb{R}^D$. Standard RoPE~\citep{su2021roformer} acts on $D/2$ 2D rotation planes. In plane $d \in \{1, \ldots, D/2\}$, the 2D slice $[\vq_{m,2d-1}, \vq_{m,2d}]^\top$ is multiplied by a rotation matrix
\begin{equation}
\Rot{\phi_m} = \begin{bmatrix} \cos\phi_m & -\sin\phi_m \\ \sin\phi_m & \cos\phi_m \end{bmatrix}.
\end{equation}
The resulting dot product between rotated query and key in plane $d$ is
\begin{equation}
(\vq'_{m,d})^\top \vk'_{n,d} = (\vq_{m,d})^\top \Rot{\phi_n - \phi_m} \vk_{n,d},
\end{equation}
which depends only on the angle difference $(\phi_n - \phi_m)$. Integer indices set $\phi_m = m\theta_d$, where $\theta_d$ is the fixed angular frequency of plane $d$. So the angle difference becomes $(n-m)\theta_d$ and the dot product depends only on the relative offset $n-m$.

\section{The T-RoPE Method}
\label{sec:method}

T-RoPE keeps RoPE's computational interface but changes the meaning of the rotation,
treating the user history as a timestamped event sequence rather than a text-like list.
The construction has five parts: real-time rotation, learned per-layer scale
coefficients, multiple temporal scales across dimensions, query alignment to the
prediction time, and a key rotation that breaks the relative-time symmetry hiding
calendar phase.

\subsection{Timestamp-Based Rotation}
\label{sec:timestamp}

Adjacent sequence positions are not equally spaced in time. In recommendation,
neighboring events can be seconds, days, or months apart, yet the index $m$ treats every
gap as an offset of one position (Figure~\ref{fig:overview}a). T-RoPE therefore
replaces the integer index with the absolute Unix timestamp $t_m$ in seconds, used
directly before rotation without centering or normalizing. As
Figure~\ref{fig:overview}b shows, events close in sequence index but far apart in time
no longer occupy neighboring phases, while bursty events stay clustered. For dimension
$d$, the angle is $\phi_{m,d} = t_m/\beta_d$, where $\beta_d$ is a base time frequency. Under this substitution the relative phase between two events becomes a function of
their true temporal gap,
\begin{equation}
(\vq'_{m,d})^\top \vk'_{n,d} \propto \cos\!\left(\frac{t_n - t_m}{\beta_d}\right),
\label{eq:timestamp_gap}
\end{equation}
so a large gap $|t_n - t_m|$ weakens attention even between adjacent events in the truncated history.

\subsection{Learnable Temporal Coefficients}
\label{sec:learnable}

Not all temporal scales deserve equal weight. Dominant cycles differ by domain and across
layers within one model. Grocery purchases may follow cycles of 7 or 30 days, fashion
demand may be seasonal, and marketplace activity may follow payday or promotion intervals.
T-RoPE learns the strength of each cycle rather than fixing it.

To learn these strengths, T-RoPE introduces learnable coefficients
$\alpha^q_d, \alpha^k_d \in \mathbb{R}$, separate for queries and keys,
\begin{equation}
\phi^q_{m,d} = \alpha^q_d \cdot \frac{t_m}{\beta_d}, \qquad
\phi^k_{n,d} = \alpha^k_d \cdot \frac{t_n}{\beta_d}.
\label{eq:learnable}
\end{equation}
Independent query and key coefficients permit asymmetric specialization, but in practice we find the opposite. Across most dimensions the learned query and key coefficients align in both sign and magnitude, as Figure~\ref{fig:multigranular}a shows, so both emphasize the same temporal bands. This alignment does not carry across layers, however. Figure~\ref{fig:layer_atten} shows the per-layer RMS coefficient magnitude varying widely (a), and the temporal attention on the same sequence shifting substantially across the selected layers (b), so each layer learns its own temporal sensitivity rather than a single shared bias.

\begin{figure}[t]
    \vspace{-15pt}
    \centering
    \includegraphics[width=0.95\linewidth,trim={7 2 3 5},clip]{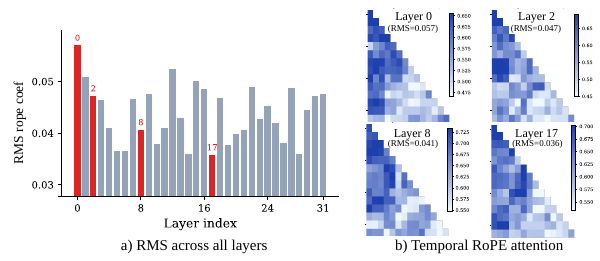}
    \vspace{-8pt}
    \caption{Layer-wise visualization of T-RoPE. a) Per-layer RMS magnitude of learned query coefficients, with the selected layers highlighted. b) Temporal RoPE attention components for selected layers on a synthetic event sequence. Larger RMS values correspond to stronger temporal modulation.} 
    \label{fig:layer_atten}
\end{figure}

\begin{figure}[!b]
    \centering
    \includegraphics[width=0.95\linewidth,trim={10 6 9 12},clip]{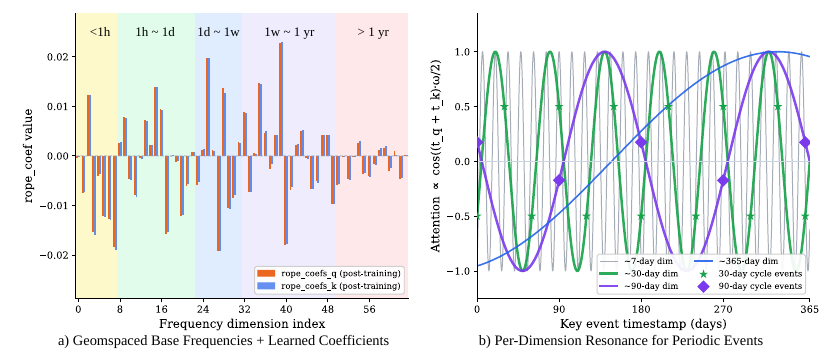}
    \vspace{-8pt}
    \caption{Multigranular temporal encoding. a) Learned query and key temporal coefficients over frequency dimensions; shaded regions group dimensions by time period, from subhour to multiyear scales. Training upweights selected temporal bands rather than all frequencies uniformly. 
    b) Representative resonance curves for those bands, where short-period dimensions oscillate rapidly while 30-day, 90-day, and yearly dimensions align with recurring user cycles.}
    \label{fig:multigranular}
\end{figure}

\subsection{Multiscale Frequency Bank}
\label{sec:freqbank}

A single base frequency resolves only one time scale, but user behavior spans many, from browsing sessions over minutes and household purchases over weeks or months to seasonal demand over a year.
Capturing behavior across these scales requires several frequencies at once.

T-RoPE assigns a distinct base frequency $\beta_d$ to each rotation plane $d$ using
geometric spacing,
\begin{equation}
\beta_d = \exp\!\left[\log\beta_{\min} + \frac{d-1}{D/2 - 1}\left(\log\beta_{\max} - \log\beta_{\min}\right)\right], \quad d = 1, \ldots, D/2,
\label{eq:geomspace}
\end{equation}
where $\beta_{\min}$ and $\beta_{\max}$ are hyperparameters (defaults $10^{2}$ and
$10^{8}$ seconds), giving a logarithmic bank from $10^{-2}$ to $10^{-8}$ radians/second.
The bank fixes the available frequencies, and the learnable coefficients from Section~\ref{sec:learnable} set how strongly each one contributes. As Figure~\ref{fig:multigranular} shows, even within a single layer training upweights a few temporal bands rather than all frequencies uniformly (a), and different frequencies lock onto different user cycles (b). The learned coefficients adjust the multiscale frequency bank so that the timestamp encoding can represent periodic patterns without explicitly encoding named calendar events such as Christmas.

\subsection{Shifted Query Alignment}
\label{sec:shifted}
\begin{wrapfigure}{r}{0.5\textwidth}
    \centering
    \vspace{-48pt}
    \includegraphics[width=0.53\textwidth,trim={5 3 1 0},clip]{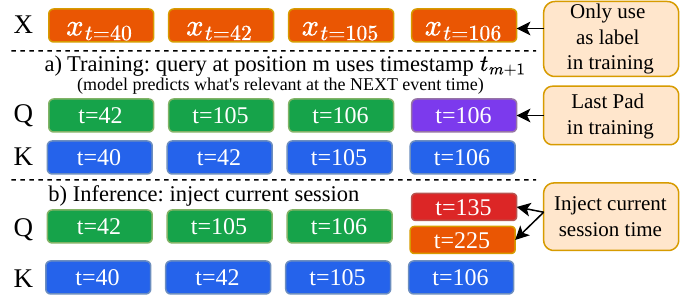}
    \vspace{-18pt}
    \caption{Shifted query alignment. a) During training, each query at position $m$ uses the next-event timestamp $t_{m+1}$, while the key keeps its historical timestamps $t_m$; the last slot is used only as a label. b) During inference, the final query slot is assigned the current session time $t_{\text{now}}$. Injecting different current times changes the temporal phase of the query while leaving the item history unchanged.}
    \label{fig:shift_query}
    \vspace{-5pt}
\end{wrapfigure}

The rotation so far encodes when each history event occurred but not when the predicted event will occur. Since position $m$ is supervised to predict the event at $t_{m+1}$ (and a deployed model receives the current time $t_{\text{now}}$ at inference), we adopt HSTU's RAB~\cite{zhai2024hstu} and rotate the query at this target time rather than at the history time.
As illustrated in Figure~\ref{fig:overview}c, T-RoPE assigns $t_{m+1}$ to the query at position $m$ while keys keep their own timestamps, 
\begin{equation}
\phi^q_{m,d} = \alpha^q_d \frac{t_{m+1}}{\beta_d}, \quad \phi^k_{n,d} = \alpha^k_d \frac{t_n}{\beta_d},
\label{eq:shifted}
\end{equation}
which is implemented as a one-step shift of the timestamp-angle tensor along the sequence
dimension. As Figure~\ref{fig:shift_query} shows, the final
training position keeps its own timestamp since it serves only as a label, while at inference
the final query slot instead receives the current session time $t_{\text{now}}$, so the same
item history can be evaluated under different prediction seasons 
(e.g., a holiday season versus a routine-replenishment context) 
without changing the architecture.

\subsection{Non-Stationary Key Encoding}
\label{sec:nonstationary}

The preceding components give RoPE real timestamps, multiple scales, and the prediction-time query, but they keep one symmetry from standard RoPE. Because queries and keys rotate the same way, the dot product depends only on the angle difference $\phi^q_m - \phi^k_n = (t_{m+1} - t_n)/\beta_d$, which a constant time shift leaves unchanged. The model can tell an item occurred seven days before the query but not where the query falls in the calendar cycle. T-RoPE removes this time-translation invariance next.

T-RoPE breaks this symmetry by applying the \emph{transpose} of the standard rotation
matrix to keys. For plane $d$ and key timestamp $t_n$,
$\vk'_{n,d} = \Rot{-\phi^k_{n,d}}\, \vk_{n,d} = \Rot{\phi^k_{n,d}}^\top \vk_{n,d}$.
Since $\Rot{-\phi} = \Rot{\phi}^\top$ for 2D rotation matrices, this keeps the same
norm-preserving rotation primitive while changing the phase relation inside attention.
The resulting dot product is
\begin{equation}
(\vq'_{m,d})^\top \vk'_{n,d}
= \vq_{m,d}^\top \Rot{-\phi^q_{m,d}} \Rot{-\phi^k_{n,d}}\, \vk_{n,d}
= \vq_{m,d}^\top \Rot{-(\phi^q_{m,d} + \phi^k_{n,d})}\, \vk_{n,d}.
\label{eq:sum_phase}
\end{equation}
The angle now depends on the \emph{sum} $(t_{m+1} + t_n)/\beta_d$ rather than the gap, so attention shifts when the same history falls in a different calendar region. We prove in Section~\ref{sec:theory} that this is sufficient for absolute-time sensitivity.

\subsection{Complete T-RoPE Formulation}

Combining the components, each head receives a timestamp-aware, multiscale,
prediction-aligned, and non-stationary rotation. The final T-RoPE rotated vectors are
\begin{equation}
\vq'_m = \bigoplus_{d=1}^{D/2} \Rot{\phi^q_{m,d}}\, \vq_{m,d}, \qquad
\vk'_n = \bigoplus_{d=1}^{D/2} \Rot{-\phi^k_{n,d}}\, \vk_{n,d},
\label{eq:full_rotation}
\end{equation}
where $\oplus$ denotes concatenation over planes, giving the attention score
\begin{equation}
a_{mn} = \sum_{d=1}^{D/2} \vq_{m,d}^\top \Rot{-(\alpha^q_d t_{m+1} + \alpha^k_d t_n)/\beta_d}\, \vk_{n,d}.
\label{eq:attention}
\end{equation}
This learned linear combination of the prediction timestamp $t_{m+1}$ and the historical
timestamp $t_n$ across all planes lets T-RoPE represent both local recency and repeated
temporal structure. 

Due to space limits, the full forward and backward implementation of T-RoPE is deferred to
Appendix~\ref{app:impl-algos}. The theoretical FLOP overhead (Appendix~\ref{app:flops}) and
measured latency (Appendix~\ref{app:latency}) show that T-RoPE adds only cost linear in
sequence length and head dimension. 

\section{Theoretical Analysis}
\label{sec:theory}

We now isolate the property the method is meant to fix, namely whether the attention score can
change when the same relative history is moved to a different calendar time, since a model that sees only temporal gaps can learn recency but cannot distinguish two queries at different phases of the same annual cycle. The argument proceeds in three steps. We first define time-translation invariance as the property that shifting every timestamp equally leaves attention unchanged; show that standard RoPE stays invariant even on timestamps when query and key share the same relative geometry; and show that T-RoPE's non-stationary key breaks this symmetry, making periodic calendar effects representable. Full proofs and the supporting remarks are deferred to Appendix~\ref{app:proofs}.

\begin{definition}[Time-translation invariance]
A positional encoding is \emph{time-translation invariant} if, for all fixed item
embedding vectors $\vq, \vk$ and all constant offsets $c \in \mathbb{R}$,
\[
\text{Attn}(\vq, \vk; t_{m+1}, t_n) = \text{Attn}(\vq, \vk; t_{m+1} + c,\, t_n + c).
\]
\end{definition}

This definition captures the exact limitation we want to avoid. If it holds, attention
can depend on the elapsed time between query and key, but it cannot depend on the
absolute calendar phase shared by both timestamps.

\begin{theorem}[Shared-coefficient timestamp RoPE is time-translation invariant]
\label{thm:relative}
Let RoPE be applied with any fixed base frequencies $\{\beta_d\}$ and any learnable
shared coefficients $\alpha_d = \alpha^q_d = \alpha^k_d$. Suppose the timestamp-based
angles are
$\phi^q_{m,d} = \alpha_d\, t_{m+1}/\beta_d$ and
$\phi^k_{n,d} = \alpha_d\, t_n/\beta_d$, and the key uses the standard rotation
$\vk'_{n,d} = \Rot{\phi^k_{n,d}}\vk_{n,d}$. Then
$\mathrm{Attn}(\vq, \vk;\, t_{m+1}, t_n)$ depends only on
$(t_{m+1} - t_n)$.
\end{theorem}

The theorem explains why timestamp substitution alone is insufficient. In each plane the
query--key dot product reduces to a rotation by $\alpha_d(t_n - t_{m+1})/\beta_d$, which
depends only on the gap and is unchanged by a common shift. Timestamp RoPE can measure how
far a past event is from the prediction time, but it still erases the shared calendar
location of the two timestamps. 

\begin{remark}[Independent standard-RoPE coefficients]
\label{rem:independent_standard_rope}
The same phase analysis shows that unequal coefficients can also break the symmetry, so the
condition above is not the only route. Independent query and key coefficients can also break
invariance. Our claim is only that T-RoPE's non-stationary key is
a RoPE-compatible \emph{sufficient} mechanism that does so directly.
\end{remark}

\begin{theorem}[T-RoPE with non-stationary keys has absolute-time sensitivity]
\label{thm:absolute}
Under T-RoPE, each per-plane summand of the attention score (Eq.~\ref{eq:attention})
rotates by the angle $-(\alpha^q_d\, t_{m+1} + \alpha^k_d\, t_n)/\beta_d$, which combines
$t_{m+1}$ and $t_n$ additively. If the trained model satisfies $\alpha^q_d + \alpha^k_d \neq 0$ for dimension
$d$, then T-RoPE is \emph{not} time-translation invariant. For any shift $c$ satisfying
$(\alpha^q_d + \alpha^k_d)c/\beta_d \notin 2\pi\mathbb{Z}$, there exist timestamps
$t_{m+1}, t_n$ and query and key vectors such that
$\mathrm{Attn}(t_{m+1}, t_n) \neq
\mathrm{Attn}(t_{m+1} + c, t_n + c)$.
\end{theorem}

Unlike Theorem~\ref{thm:relative}, the non-stationary key turns the per-plane phase into
the \emph{sum} $(\alpha^q_d t_{m+1} + \alpha^k_d t_n)/\beta_d$ rather than a difference, so
a common shift moves it by $(\alpha^q_d + \alpha^k_d)c/\beta_d$ and the score changes
unless this is a multiple of $2\pi$. Because the coefficients are learnable per dimension,
the model chooses which timescales stay stationary and which become absolute-time
sensitive, and in our deployed model this induces visible layer-dependent temporal bias
(Figure~\ref{fig:layer_atten}) rather than the degenerate
$\alpha^q_d=-\alpha^k_d$ case. Absolute-time sensitivity matters because many
recommendation signals are periodic, which the next corollary connects to calendar cycles.

\begin{corollary}[Seasonality representability]
\label{cor:seasonal}
Let $P > 0$ be a temporal period (e.g., 365 days), and define the effective query
angular frequency $\omega^q_d = \alpha^q_d/\beta_d$. T-RoPE can represent a
sinusoidal dependence on the prediction-time calendar phase with period $P$ if there
exists a dimension $d$ and nonzero integer harmonic $r$ such that
$\omega^q_d = 2\pi r/P$.
\end{corollary}

Intuitively, with the key timestamp fixed the plane-$d$ score carries a
$\cos(\alpha^q_d t_{m+1}/\beta_d + \text{const})$ term, so choosing $\omega^q_d = 2\pi r/P$
makes attention periodic in the prediction time with period $P$. 

\section{Experiments}
\label{sec:experiments}

We evaluate T-RoPE in three settings. First, public benchmarks across domains and sizes test
general accuracy. Second, a Shop ablation isolates which temporal mechanisms matter under
non-stationary commerce. Third, an online A/B test measures production retrieval and business impact.

\subsection{Validating T-RoPE on Public Datasets}
\label{sec:public_experiments}

To validate T-RoPE against existing methods with reproducible results, we evaluate on five
public recommendation benchmarks whose statistics are listed in
Table~\ref{tab:datasets} (Appendix~\ref{app:datasets}). We compare against five existing
methods: \textbf{SASRec}~\citep{kang2018self} (no temporal encoding),
\textbf{TiSASRec}~\citep{li2020time} (time-interval attention biases),
\textbf{HSTU}~\citep{zhai2024hstu} (position-only RAB),
\textbf{HSTU + Time RAB}~\citep{zhai2024hstu}, and
\textbf{HSTU + TO-RoPE}~\citep{wei2025torope} (HSTU with Time+Pos RAB + TO-RoPE). All of these, together with our
\textbf{HSTU + T-RoPE} (HSTU with Time+Pos RAB and T-RoPE), are
implemented in the same PyTorch training framework and share identical item-ID inputs,
preprocessing filters, maximum sequence lengths, leave-one-out protocol, candidate
evaluation procedure, and training hyperparameters, making the comparison fair and consistent.
Full hyperparameter settings are reported in Appendix~\ref{app:hyperparams}.

\begin{table*}[t]
\vspace{-10pt}
\caption{Sequential recommendation results (\%). Best per row in \textbf{bold}, second best \underline{underlined}. The \textcolor{relgreen}{\textbf{green}} value beside our result is its relative improvement (\%) over the second-best method in that row.}
\label{tab:main}
\centering
\small
\setlength{\tabcolsep}{4.8pt}
\renewcommand{\arraystretch}{1.08}
\begin{tabular}{@{}llccccc>{\columncolor{gray!20}}c@{}}
\toprule
& & \multicolumn{6}{c}{Method} \\
\cmidrule(l){3-8}
Dataset & Metric
  & SASRec
  & TiSASRec
  & HSTU
  & \makecell{HSTU+\\Time RAB}
  & \makecell{HSTU+\\TO-RoPE}
  & \makecell{\textbf{HSTU+}\\\textbf{T-RoPE (ours)}} \\
\midrule
\multirow{5}{*}{ML-20M}
  & HR@10   & 31.36 & 31.97 & 30.87 & \underline{32.05} & 31.19 & \textbf{32.37}\,\rel{1.0} \\
  & HR@50   & 56.87 & 57.50 & 56.19 & \underline{58.01} & 56.69 & \textbf{58.55}\,\rel{0.9} \\
  & NDCG@10 & 18.04 & 18.38 & 17.84 & \underline{18.47} & 17.93 & \textbf{18.62}\,\rel{0.8} \\
  & NDCG@50 & 23.68 & 24.04 & 23.44 & \underline{24.20} & 23.57 & \textbf{24.42}\,\rel{0.9} \\
  & MRR     & 15.42 & 15.68 & 15.31 & \underline{15.79} & 15.32 & \textbf{15.91}\,\rel{0.8} \\
\midrule
\multirow{5}{*}{Amazon Books}
  & HR@10   &  2.97 &  3.19 &  5.88 & \underline{ 6.89} &  6.69             & \textbf{ 7.65}\,\rel{11.0} \\
  & HR@50   &  7.26 &  7.66 & 12.40 & \underline{14.27} & 13.30             & \textbf{15.45}\,\rel{8.3} \\
  & NDCG@10 &  1.57 &  1.71 &  3.33 &  3.93             & \underline{ 3.96} & \textbf{ 4.40}\,\rel{11.1} \\
  & NDCG@50 &  2.50 &  2.68 &  4.75 & \underline{ 5.54} &  5.40             & \textbf{ 6.11}\,\rel{10.3} \\
  & MRR     &  1.41 &  1.53 &  2.94 &  3.46             & \underline{ 3.51} & \textbf{ 3.87}\,\rel{10.3} \\
\midrule
\multirow{5}{*}{Pixel200K}
  & HR@10   &  4.72 &  4.65 &  4.22 & \underline{ 7.21} &  4.16 & \textbf{ 7.41}\,\rel{2.8} \\
  & HR@50   & 10.85 & 10.86 & 10.26 & \underline{16.82} & 10.09 & \textbf{17.37}\,\rel{3.3} \\
  & NDCG@10 &  2.65 &  2.61 &  2.33 & \underline{ 4.08} &  2.30 & \textbf{ 4.15}\,\rel{1.7} \\
  & NDCG@50 &  3.98 &  3.93 &  3.63 & \underline{ 6.14} &  3.58 & \textbf{ 6.30}\,\rel{2.6} \\
  & MRR     &  2.39 &  2.35 &  2.12 & \underline{ 3.70} &  2.11 & \textbf{ 3.76}\,\rel{1.6} \\
\midrule
\multirow{5}{*}{Pixel1M}
  & HR@10   &  7.12 &  7.16 &  6.89 & \underline{12.46} &  6.81 & \textbf{13.35}\,\rel{7.1} \\
  & HR@50   & 15.62 & 15.77 & 15.31 & \underline{26.16} & 15.14 & \textbf{27.44}\,\rel{4.9} \\
  & NDCG@10 &  4.04 &  4.05 &  3.88 & \underline{ 7.22} &  3.83 & \textbf{ 7.81}\,\rel{8.2} \\
  & NDCG@50 &  5.88 &  5.92 &  5.70 & \underline{10.19} &  5.64 & \textbf{10.88}\,\rel{6.8} \\
  & MRR     &  3.61 &  3.62 &  3.46 & \underline{ 6.42} &  3.43 & \textbf{ 6.95}\,\rel{8.3} \\
\midrule
\multirow{5}{*}{Pixel8M}
  & HR@10   &  4.89 &  4.96 &  5.01 & \underline{ 9.14} &  4.99 & \textbf{11.47}\,\rel{25.5} \\
  & HR@50   & 10.69 & 10.82 & 10.93 & \underline{18.87} & 10.86 & \textbf{23.02}\,\rel{22.0} \\
  & NDCG@10 &  2.75 &  2.79 &  2.82 & \underline{ 5.36} &  2.82 & \textbf{ 6.85}\,\rel{27.8} \\
  & NDCG@50 &  4.00 &  4.06 &  4.11 & \underline{ 7.47} &  4.09 & \textbf{ 9.35}\,\rel{25.2} \\
  & MRR     &  2.44 &  2.48 &  2.51 & \underline{ 4.78} &  2.50 & \textbf{ 6.11}\,\rel{27.8} \\
\bottomrule
\end{tabular}
\end{table*}

Detailed results across the five benchmarks are reported in Table~\ref{tab:main}, covering HR@10, HR@50, NDCG@10, NDCG@50, and MRR. HSTU + T-RoPE achieves the best result on all five metrics on every dataset, and the improvement scales with how much temporal structure the data contain. On the dense
ML-20M benchmark, whose histories are fairly regular, gains are modest but consistent. The
effect is clearer on sparse and commercial data, where T-RoPE improves over HSTU + Time RAB by 8--12\%
across all five metrics on Amazon Books, and over TO-RoPE by 78--130\% in HR@10 and
72--112\% in HR@50 across the PixelRec splits. The largest margins appear on the sparse and commercial data, where absolute calendar phase carries the most signal.

\subsection{Ablation Study on an Industrial-Scale Dataset}
\label{sec:ablation}

To identify which design choices drive the gains, we ablate T-RoPE components on Shop, an
internal e-commerce dataset with roughly 150M users, 4.8M items, and over 6B interactions
spanning multiple years. Larger and more temporally heterogeneous than the public
benchmarks, it better tests the multiscale and non-stationary mechanisms. The model uses
128-dimensional embeddings, 12 HSTU blocks, 8 attention heads, 32 attention dimensions per
head, and a maximum sequence length of 512. We use a one-day temporal holdout (training on
interactions before the holdout date, every meaningful interaction on the held-out day
counting as a positive) and report recall@K rather than HR@K, since a user may have
several relevant interactions. The detailed results are shown in Table~\ref{tab:ablation}, which removes components cumulatively from the full T-RoPE model down to the HSTU + Time RAB backbone, with the key findings discussed below.

\definecolor{impgreen}{RGB}{0,128,0}
\definecolor{gapred}{HTML}{D1242F}
\newcommand{\dgrn}[1]{{\fontsize{6}{7}\selectfont\textcolor{impgreen}{\textbf{(+#1\%)}}}}
\newcommand{\drd}[1]{{\fontsize{6}{7}\selectfont\textcolor{gapred}{\textbf{($-$#1\%)}}}}

\begin{table}[t]
\vspace{-10pt}
\caption{Ablation study on the internal Shop e-commerce dataset.
The shaded top row is the full T-RoPE model. Rows~2--5 remove components cumulatively down
to the HSTU + Time RAB backbone, and rows~6--8 remove single components from the full model.
Metric values are percentages, with the best value in each column shown in \textbf{bold}.
Parenthetical values show each row's relative change from the full model.
\textcolor{impgreen}{Green} indicates an improvement over T-RoPE, while
\textcolor{gapred}{red} indicates a shortfall.}
\label{tab:ablation}
\centering
\small
\setlength{\tabcolsep}{5pt}
\renewcommand{\arraystretch}{0.95}
\begin{tabular}{@{}lrrrrr@{}}
\toprule
Configuration & Recall@50 & NDCG@50 & Recall@200 & NDCG@200 & MRR \\
\midrule
\rowcolor{gray!20}
\textbf{T-RoPE}
  & \textbf{31.41} & \textbf{22.94} & 41.77 & \textbf{24.82} & \textbf{23.14} \\
~~- Non-Stationary (§\ref{sec:nonstationary})
  & \makecell[r]{30.97\\[-2.5pt]\drd{1.40}} & \makecell[r]{22.02\\[-2.5pt]\drd{4.01}}
  & \makecell[r]{40.49\\[-2.5pt]\drd{3.06}} & \makecell[r]{23.29\\[-2.5pt]\drd{6.16}}
  & \makecell[r]{21.88\\[-2.5pt]\drd{5.45}} \\
~~- shifted (§\ref{sec:shifted})
  & \makecell[r]{30.90\\[-2.5pt]\drd{1.62}} & \makecell[r]{20.59\\[-2.5pt]\drd{10.24}}
  & \makecell[r]{40.92\\[-2.5pt]\drd{2.03}} & \makecell[r]{22.55\\[-2.5pt]\drd{9.15}}
  & \makecell[r]{20.90\\[-2.5pt]\drd{9.68}} \\
~~- multiscale freq.~bank (§\ref{sec:freqbank})
  & \makecell[r]{25.36\\[-2.5pt]\drd{19.26}} & \makecell[r]{13.23\\[-2.5pt]\drd{42.33}}
  & \makecell[r]{36.98\\[-2.5pt]\drd{11.47}} & \makecell[r]{15.42\\[-2.5pt]\drd{37.87}}
  & \makecell[r]{11.96\\[-2.5pt]\drd{48.31}} \\
~~- abs.\ time embedding (§\ref{sec:timestamp} + §\ref{sec:learnable})
  & \makecell[r]{25.68\\[-2.5pt]\drd{18.24}} & \makecell[r]{13.81\\[-2.5pt]\drd{39.80}}
  & \makecell[r]{37.00\\[-2.5pt]\drd{11.42}} & \makecell[r]{15.96\\[-2.5pt]\drd{35.70}}
  & \makecell[r]{12.70\\[-2.5pt]\drd{45.12}} \\
\midrule
~~- Shifted Query RoPE
  & \makecell[r]{30.87\\[-2.5pt]\drd{1.72}} & \makecell[r]{21.54\\[-2.5pt]\drd{6.10}}
  & \makecell[r]{40.50\\[-2.5pt]\drd{3.04}} & \makecell[r]{23.42\\[-2.5pt]\drd{5.64}}
  & \makecell[r]{22.46\\[-2.5pt]\drd{2.94}} \\
\midrule
~~- Key RoPE
  & \makecell[r]{31.19\\[-2.5pt]\drd{0.70}} & \makecell[r]{20.58\\[-2.5pt]\drd{10.29}}
  & \makecell[r]{\textbf{41.90}\\[-2.5pt]\dgrn{0.31}} & \makecell[r]{22.63\\[-2.5pt]\drd{8.82}}
  & \makecell[r]{20.90\\[-2.5pt]\drd{9.68}} \\
\midrule
~~- Time RAB
  & \makecell[r]{29.08\\[-2.5pt]\drd{7.42}} & \makecell[r]{20.76\\[-2.5pt]\drd{9.50}}
  & \makecell[r]{39.69\\[-2.5pt]\drd{4.98}} & \makecell[r]{22.95\\[-2.5pt]\drd{7.53}}
  & \makecell[r]{16.89\\[-2.5pt]\drd{27.01}} \\
\bottomrule
\end{tabular}
\end{table}

\paragraph{The multiscale frequency bank and shifted queries supply the core gains.}
The multiscale bank has the largest effect. Removing it cuts NDCG@50 by 36\% when comparing
rows~3 and~4. This shows that multiscale rotation turns raw timestamps into usable frequency
structure. Without it, the timestamp is just noise, and the single frequency embedding in
row~4 falls about 4\% below the HSTU + Time RAB backbone in row~5. Shifted query alignment
provides the next largest gain by adding seasonality to the query at inference. It improves
NDCG@50 by 7\% when comparing rows~2 and~3.

\paragraph{The non-stationary key drives absolute-time sensitivity, complemented by key and query rotation.}
Adding the non-stationary key (§\ref{sec:nonstationary}, row~1) outperforms the shifted model (row~2) on all five metrics, improving recall@200 by 3.2\% and NDCG@50 by 4.2\%, which confirms that it injects useful absolute-time sensitivity. Key and query rotation
each carry such sensitivity on their own (Remark~\ref{rem:independent_standard_rope}), so
we further ablate them one at a time. The shifted-query rotation has the larger effect on
retrieval, since removing it (row~6) lowers recall@200 by 3.0\%, whereas removing the key
rotation (row~7) instead costs ranking precision, cutting NDCG@50 by 10.3\%. The two are therefore complementary, and the
non-stationary key combines both sides and pushes furthest.

\paragraph{Time RAB and T-RoPE are complementary.}

Removing Time RAB from the full model (row~8) lowers all metrics, with the starkest drop of 27.0\% in MRR.  This suggests RAB still captures useful relative-recency signals that T-RoPE does not. Even so, T-RoPE without RAB still improves NDCG@50 by roughly 50\% over the baseline, so the two mechanisms contribute independently and combine for the best result.



\subsection{Online A/B Test}
\label{sec:online_ab}

To measure real production impact, we deploy T-RoPE in an online A/B test in the Shop app
recommender, scaling up the offline architecture while keeping the T-RoPE design unchanged.
Table~\ref{tab:online_ab} reports relative lift against the production control,
an identical model without T-RoPE.

\begin{table}[t]
\vspace{-10pt}
\caption{Online A/B test results in the Shop app. Values are relative lifts against the control group.}
\label{tab:online_ab}
\centering
\small
\setlength{\tabcolsep}{6pt}
\begin{tabular}{@{}llrr@{}}
\toprule
Metric & Definition & Relative lift & P-Value \\
\midrule
CVR & Conversion rate & $+0.33\%$ & $0.037$ \\
Shop orders & Count of orders placed on the Shop app & $+0.63\%$ & $0.094$ \\
\bottomrule
\end{tabular}
\end{table}

The online results align with the offline gains. CVR shows the strongest signal
($+0.33\%$, $p=0.037$), indicating that time-aware retrieval converts more sessions into
interactions. This carries to higher-intent behavior. 
The Shop orders also increase $+0.63\%$ ($p=0.094$), directional evidence of downstream value. Overall, time-aware RoPE
yields positive business gains in production.

\section{Related Work}
\label{sec:related}

\paragraph{Sequential recommendation models}
Sequential recommendation models user preferences from ordered interaction histories.
Early deep models used recurrent networks (GRU4Rec~\citep{hidasi2016session},
NARM~\citep{li2017narm}) and convolution (Caser~\citep{tang2018caser}). The
Transformer~\citep{vaswani2017attention} reshaped the field: SASRec~\citep{kang2018self}
showed that unidirectional self-attention outperforms RNNs, and
BERT4Rec~\citep{sun2019bert4rec} added bidirectional masked-item prediction. Later work
enriched attention with item attributes (FDSA~\citep{zhang2019fdsa}), replaced it with
learnable frequency-domain filters (FMLP-Rec~\citep{zhou2022fmlp}), or added
self-supervised pretraining (S$^3$-Rec~\citep{zhou2020s3rec}); at production scale,
HSTU~\citep{zhai2024hstu} reaches trillions of parameters via hierarchical action-based
transduction. A newer thread casts recommendation as autoregressive generation over
discrete semantic tokens: TIGER~\citep{rajput2023tiger} uses RQ-VAE item codes,
LETTER~\citep{wang2024letter} adds collaborative signals to tokenization,
ActionPiece~\citep{hou2025actionpiece} tokenizes at action granularity,
PSID~\citep{zhang2025psid} relies on pure semantic IDs, and HLLM~\citep{chen2024hllm,zhang2026grlm,zhang2026ragr,liu2026omega,he2026pauserec}
pairs an item-level with a user-level LLM. T-RoPE is a drop-in positional encoding
compatible with any of these Transformer backbones.

\paragraph{Temporal information in sequential recommendation}
Temporal dynamics are a core problem in sequential recommendation. Early collaborative
filtering showed that preferences and item popularity shift substantially over calendar
time~\citep{koren2009temporal}. TiSASRec~\citep{li2020time} discretizes the elapsed gap
$|t_i - t_j|$ into buckets injected as additive attention biases, capturing relative
recency but only the gap magnitude, with no calendar-phase awareness. Knowledge-guided
methods~\citep{wang2020chorus} learn time-decay kernels over item relations, and
attention-mixture approaches~\citep{tran2023mojito} condition on calendar features such as
hour of day. TO-RoPE~\citep{wei2025torope} blends position indices and timestamps inside
RoPE angles, making attention sensitive to temporal distance, but because query and key
share the same timestamp coefficient it remains time-translation invariant
(Theorem~\ref{thm:relative}, Remark~\ref{rem:torope}). T-RoPE instead introduces a
non-stationary key with learnable query/key coefficients, breaking that invariance and
giving the absolute-time awareness needed for periodic patterns such as seasonal purchase
cycles.



\section{Conclusion}
\label{sec:conclusion}

We introduced T-RoPE, a time-aware rotary embedding that rotates attention by when events happen rather than by their order, so the model reasons directly over elapsed time and recurring calendar cycles. We proved that shared-geometry timestamp encodings are time-translation invariant and cannot model seasonal preferences, motivating T-RoPE's non-stationary keys. We validated it on public benchmarks and a large commercial dataset, where it consistently outperforms baselines, with the largest gains under high temporal heterogeneity. Grounding rotary attention in time is thus principled and effective for recommendation.

\subsection*{AI use statement}

We used AI only to aid and polish writing at the sentence level,
including edits to clarity, grammar, and phrasing. We did not use generative AI
for literature retrieval or discovery, research ideation, theoretical
development, methodology or experiment design, code implementation,
data generation or processing, data analysis, or interpretation of results.
The authors reviewed every AI-assisted edit and take full responsibility for
the final content of this work.

\subsection*{Ethics statement}

Public experiments used established recommendation benchmarks obtained from the cited sources. The industrial experiments use data for which the authors had authorized access, and the online experiment compared two recommender variants in the normal operation of the Shop app. The industrial and online studies followed the company’s processes for data access, privacy, and experimentation governance. We report their results only in aggregate and release neither individual user records nor personally identifying information. The proprietary data and production infrastructure therefore cannot be made public. 

\subsection*{Reproducibility statement}

We support reproduction of the public-benchmark results through the complete
method specification in Section~\ref{sec:method}, the experimental protocol in
Section~\ref{sec:public_experiments}, and the implementation details in
Appendix~\ref{app:impl}. Appendix~\ref{app:impl-algos} gives forward and backward
algorithms, Appendix~\ref{app:datasets} lists the public datasets and their
statistics, and Appendix~\ref{app:hyperparams} reports the architecture,
optimization settings, random seed, and hardware. Appendix~\ref{app:proofs}
contains full proofs of the theoretical claims, while
Appendices~\ref{app:flops} and~\ref{app:latency} document computational cost and
measured throughput.

\bibliography{iclr2027_conference}
\bibliographystyle{iclr2027_conference}

\clearpage
\appendix
\section*{Appendix}

\section{Full Proofs}
\label{app:proofs}

\subsection{Proof of Theorem~\ref{thm:relative}}

\emph{Restatement.} With shared coefficients $\alpha_d=\alpha^q_d=\alpha^k_d$ and angles
$\phi^q_{m,d}=\alpha_d t_{m+1}/\beta_d$, $\phi^k_{n,d}=\alpha_d t_n/\beta_d$ under the
standard key rotation $\vk'_{n,d}=\Rot{\phi^k_{n,d}}\vk_{n,d}$, the attention score depends
only on $(t_{m+1}-t_n)$ and is therefore time-translation invariant.

\begin{proof}
We give the full argument for standard key rotation with shared query and key timestamp
coefficients $\alpha^q_d=\alpha^k_d\triangleq\alpha_d$.

Standard RoPE with $\phi^q_{m,d} = \alpha_d t_{m+1}/\beta_d$ and
$\phi^k_{n,d} = \alpha_d t_n/\beta_d$ (standard key rotation
$\Rot{\phi^k_{n,d}}$) gives
\begin{align*}
(\vq'_{m,d})^\top \vk'_{n,d}
&= \vq_{m,d}^\top \Rot{\phi^k_{n,d} - \phi^q_{m,d}}\, \vk_{n,d}
= \vq_{m,d}^\top \Rot{\alpha_d(t_n - t_{m+1})/\beta_d}\, \vk_{n,d},
\end{align*}
which depends only on $(t_n - t_{m+1})$. Summing over planes and heads preserves
time-translation invariance.

For comparison, if standard key rotation uses independent coefficients
$\alpha^q_d$ and $\alpha^k_d$, the phase is
\[
\phi^k_{n,d}-\phi^q_{m,d}
= (\alpha^k_d t_n-\alpha^q_d t_{m+1})/\beta_d .
\]
Under a global shift by $c$, the phase changes by
$(\alpha^k_d-\alpha^q_d)c/\beta_d$. Therefore this standard-RoPE family is invariant
only for active planes where $\alpha^q_d=\alpha^k_d$. Opposite nonzero coefficients
depend on the timestamp sum and are not time-translation invariant.
\end{proof}


\begin{remark}[Relation to TO-RoPE]
\label{rem:torope}
The same argument covers the timestamp component of TO-RoPE, which shares one timestamp
coefficient across query and key. That component depends only on $(\tau_m-\tau_n)$ and is
therefore time-translation invariant, so it cannot represent seasonal shifts on its own.
\end{remark}

\subsection{Proof of Theorem~\ref{thm:absolute}}

\emph{Restatement.} Under T-RoPE the per-plane attention equals
$\vq_{m,d}^\top \Rot{-(\alpha^q_d t_{m+1} + \alpha^k_d t_n)/\beta_d}\, \vk_{n,d}$; if
$\alpha^q_d + \alpha^k_d \neq 0$ for some dimension $d$, then T-RoPE is not
time-translation invariant.

\begin{proof}
We expand Eq.~\ref{eq:sum_phase} carefully using the rotation group property $\Rot{a}\Rot{b} = \Rot{a+b}$,
\begin{align*}
(\vq'_{m,d})^\top \vk'_{n,d}
&= \left(\Rot{\phi^q_{m,d}} \vq_{m,d}\right)^\top \left(\Rot{-\phi^k_{n,d}} \vk_{n,d}\right) \\
&= \vq_{m,d}^\top \Rot{\phi^q_{m,d}}^\top \Rot{-\phi^k_{n,d}}\, \vk_{n,d} \\
&= \vq_{m,d}^\top \Rot{-\phi^q_{m,d}} \Rot{-\phi^k_{n,d}}\, \vk_{n,d} \\
&= \vq_{m,d}^\top \Rot{-\phi^q_{m,d} - \phi^k_{n,d}}\, \vk_{n,d}.
\end{align*}
Substituting $\phi^q_{m,d} = \alpha^q_d t_{m+1}/\beta_d$ and $\phi^k_{n,d} = \alpha^k_d t_n/\beta_d$ gives
\[
(\vq'_{m,d})^\top \vk'_{n,d}
= \vq_{m,d}^\top \Rot{-(\alpha^q_d t_{m+1} + \alpha^k_d t_n)/\beta_d}\, \vk_{n,d}.
\]
Under the time shift $(t_{m+1}, t_n) \to (t_{m+1}+c, t_n+c)$, the argument changes by $(\alpha^q_d + \alpha^k_d)c/\beta_d$. If $(\alpha^q_d + \alpha^k_d)c/\beta_d \notin 2\pi\mathbb{Z}$, the shifted and unshifted rotation matrices are different. Since two different rotation matrices define different bilinear forms, there exist nondegenerate vectors $\vq_{m,d}$ and $\vk_{n,d}$ for which the dot product changes after the shift. This gives a counterexample to the universal equality required by time-translation invariance. If the phase shift is in $2\pi\mathbb{Z}$, that particular shift can coincide, but such modulo coincidences do not establish invariance for all shifts.
\end{proof}

\begin{remark}[Learnable degeneracy and timescale selection]
\label{rem:degeneracy}
The condition $\alpha^q_d + \alpha^k_d \neq 0$ for some $d$ is genuinely needed. If training
converges to $\alpha^q_d = -\alpha^k_d$ for all $d$, T-RoPE degenerates to the
time-translation invariant case. This is not purely a failure mode, since learnable
per-dimension coefficients let the model choose which timescales stay stationary and which
become absolute-time sensitive. In our deployed model the degenerate case does not arise,
and the learned coefficients induce visible layer-dependent temporal bias
(Figure~\ref{fig:layer_atten}).
\end{remark}

\subsection{Proof of Corollary~\ref{cor:seasonal}}

\emph{Restatement.} For a period $P>0$ and effective query angular frequency
$\omega^q_d=\alpha^q_d/\beta_d$, T-RoPE can represent a sinusoidal dependence on the
prediction-time calendar phase with period $P$ whenever $\omega^q_d=2\pi r/P$ for some
dimension $d$ and nonzero integer harmonic $r$.

\begin{proof}
The attention in dimension $d$ can be written as
\[
f_d(t_{m+1}, t_n) = \vq_{m,d}^\top \Rot{-(\alpha^q_d t_{m+1} + \alpha^k_d t_n)/\beta_d}\, \vk_{n,d}.
\]
Expanding in terms of $\cos$ and $\sin$,
\[
f_d = (q_1 k_1 + q_2 k_2)\cos\!\left(\frac{\alpha^q_d t_{m+1} + \alpha^k_d t_n}{\beta_d}\right)
     + (q_2 k_1 - q_1 k_2)\sin\!\left(\frac{\alpha^q_d t_{m+1} + \alpha^k_d t_n}{\beta_d}\right).
\]
Holding $t_n$ fixed, the key term $\alpha^k_d t_n/\beta_d$ is a constant phase offset. Let $\omega^q_d=\alpha^q_d/\beta_d$. If $\omega^q_d=2\pi r/P$ for a nonzero integer harmonic $r$, then both the sine and cosine terms are periodic in $t_{m+1}$ with period $P$. Thus T-RoPE has the capacity to express sinusoidal modulation over prediction time at the target calendar period.
\end{proof}

\section{Implementation Details}
\label{app:impl}

\subsection{Forward and backward algorithms}
\label{app:impl-algos}

Each layer using T-RoPE maintains a geometric frequency bank $\beta \in \mathbb{R}^{D/2}$ and two trainable coefficient vectors $\alpha^q,\alpha^k \in \mathbb{R}^{D/2}$. The frequency bank uses $\beta_{\min}=10^2$ seconds and $\beta_{\max}=10^8$ seconds by default, and the coefficients are initialized independently from $\mathcal{N}(0,\sigma^2)$ with $\sigma=1.0$.

Algorithm~\ref{alg:trope} expands the forward pass. Because we use absolute Unix
timestamps in seconds, which are on the order of $10^9$, T-RoPE computes $\Phi^q$ and
$\Phi^k$ in FP32 even when the surrounding model uses BF16. The rotated outputs are
cast back to the model dtype after the trigonometric functions have been evaluated.

\begin{algorithm}[t]
\caption{T-RoPE forward pass}
\label{alg:trope}
\begin{algorithmic}[1]
\Require queries $Q$, keys $K \in \mathbb{R}^{B \times L \times H \times D}$; timestamps $\mathbf{t} \in \mathbb{R}^{B \times L}$; coefficients $\alpha^q,\alpha^k \in \mathbb{R}^{D/2}$
\State $\beta \leftarrow \geomspace(10^2, 10^8, D/2)$
\State $\Phi^q \leftarrow \textsc{FP32}(\alpha^q)[\text{None},\text{None},:] \cdot \mathbf{t}[:,:, \text{None}] / \beta[\text{None},\text{None},:]$
\State $\Phi^q \leftarrow \textsc{RollLeft}(\Phi^q,\text{dim}=1)$
\State $\Phi^q[:,L,:] \leftarrow \textsc{FP32}(\alpha^q)[:] \cdot \mathbf{t}[:,L,\text{None}] / \beta[:]$
    \Comment{last query keeps its timestamp}
\State $\Phi^k \leftarrow \textsc{FP32}(\alpha^k)[\text{None},\text{None},:] \cdot \mathbf{t}[:,:, \text{None}] / \beta[\text{None},\text{None},:]$
\State $Q' \leftarrow \textsc{ApplyRotFromTheta}(Q,\Phi^q,\mathrm{standard})$
\State $K' \leftarrow \textsc{ApplyRotFromTheta}(K,\Phi^k,\mathrm{nonstationary})$
\State \Return $Q', K'$
\end{algorithmic}
\end{algorithm}

\begin{algorithm}[t]
\caption{\textsc{ApplyRotFromTheta}}
\label{alg:apply-rot-theta}
\begin{algorithmic}[1]
\Require input $X$, angles $\Phi$, rotation type $r \in \{\mathrm{standard}, \mathrm{nonstationary}\}$
\State $C \leftarrow \cos\Phi$, $S \leftarrow \sin\Phi$
\State Split rotary dimensions of $X$ into planes $(x_0,x_1)$
\If{$r=\mathrm{standard}$}
    \State $y_0 \leftarrow x_0 C - x_1 S$; \quad $y_1 \leftarrow x_0 S + x_1 C$
\Else
    \State $y_0 \leftarrow x_0 C + x_1 S$; \quad $y_1 \leftarrow -x_0 S + x_1 C$
\EndIf
\State Preserve all non-rotary tail dimensions of $X$ unchanged
\State \Return concatenated output $Y$
\end{algorithmic}
\end{algorithm}

Algorithm~\ref{alg:trope-backward} summarizes the custom autograd path used for learnable angles. A fixed RoPE backward only needs $dX$, which is obtained by applying the conjugate rotation to the upstream gradient. T-RoPE additionally returns $d\Phi$ because $\Phi$ depends on trainable coefficients. This is the main computational difference from raw RoPE.

\begin{algorithm}[t]
\caption{Backward pass for learnable T-RoPE angles}
\label{alg:trope-backward}
\begin{algorithmic}[1]
\Require upstream gradient $G$, saved input $X$, saved angles $\Phi$, rotation type $r$
\State $C \leftarrow \cos\Phi$, $S \leftarrow \sin\Phi$
\State Split rotary dimensions of $X$ and $G$ into planes $(x_0,x_1)$ and $(g_0,g_1)$
\If{$r=\mathrm{standard}$}
    \State $dx_0 \leftarrow g_0 C + g_1 S$; \quad $dx_1 \leftarrow -g_0 S + g_1 C$
    \State $d\Phi \leftarrow g_0(-x_0S - x_1C) + g_1(x_0C - x_1S)$
\Else
    \State $dx_0 \leftarrow g_0 C - g_1 S$; \quad $dx_1 \leftarrow g_0 S + g_1 C$
    \State $d\Phi \leftarrow g_0(-x_0S + x_1C) + g_1(-x_1S - x_0C)$
\EndIf
\State $d\Phi \leftarrow \textsc{SumHeads}(d\Phi)$
    \Comment{sum contributions from all heads}
\State \Return $dX, d\Phi$
\end{algorithmic}
\end{algorithm}

The gradients for the learnable coefficients follow directly from $\Phi^q_{b,l,d}=\alpha^q_d t^q_{b,l}/\beta_d$ and $\Phi^k_{b,l,d}=\alpha^k_d t_{b,l}/\beta_d$,
\begin{equation}
\frac{\partial \mathcal{L}}{\partial \alpha^q_d}
= \sum_{b,l} \frac{\partial \mathcal{L}}{\partial \Phi^q_{b,l,d}}\frac{t^q_{b,l}}{\beta_d}, \qquad
\frac{\partial \mathcal{L}}{\partial \alpha^k_d}
= \sum_{b,l} \frac{\partial \mathcal{L}}{\partial \Phi^k_{b,l,d}}\frac{t_{b,l}}{\beta_d},
\end{equation}
where $t^q_{b,l}=t_{b,l+1}$ for $l<L$ and $t^q_{b,L}=t_{b,L}$ under the shifted query boundary rule. Because $t_{b,l+1}$ is a Unix timestamp measured in seconds, the gradients are usually large and must be clipped or rescaled to stabilize training.

\subsection{Parameter and FLOP overhead}
\label{app:flops}

Let $B$ be the batch size and $L$ the sequence length, so $N=BL$ is the total number of tokens. Let $H$ be the number of attention heads, $D$ the per-head dimension, $P=D/2$ the number of rotary planes (each plane is a $2\times2$ rotation over one channel pair), and $M$ the number of Transformer layers that use T-RoPE. Because each T-RoPE module keeps separate query and key coefficients, it adds $2P=D$ learnable parameters, so the full model adds $MD$ parameters in total. In our HSTU experiments ($M=2$ layers, $D=128$), this is only $2\cdot128=256$ additional parameters.

We count multiply, add/subtract, and divide as one FLOP each, and report $\sin$ and $\cos$ separately as transcendental operations. The non-stationary key rotation has the same arithmetic count as a standard rotation. It only changes signs in the rotation equations.
The rotation cost below corresponds to Algorithm~\ref{alg:apply-rot-theta} applied once to $Q$ and once to $K$.

\begin{center}
\small
\begin{tabular}{lcc}
\toprule
Operation & Arithmetic FLOPs & Transcendental ops \\
\midrule
Build $\Phi^q,\Phi^k$ & $4NP$ & -- \\
Rotate $Q,K$ & $12NHP$ & $2NP$ $\sin$ and $2NP$ $\cos$ \\
Fixed-RoPE backward $dQ,dK$ & $12NHP$ & $2NP$ $\sin$ and $2NP$ $\cos$ \\
Extra learnable-angle $d\Phi$ & $18NHP + 2NP(H-1)$ & -- \\
Coefficient gradients & $\approx 4NP$ & -- \\
\bottomrule
\end{tabular}
\end{center}

Thus, relative to raw RoPE with a fixed precomputed rotation cache, T-RoPE adds $4NP$ forward arithmetic FLOPs to build timestamp-dependent angles, plus runtime trigonometric evaluations if the kernel computes $\sin\Phi$ and $\cos\Phi$ on the fly. During training, the extra cost over fixed raw RoPE is the $d\Phi$ computation and reduction, followed by the coefficient-gradient chain above. These extra terms are linear in sequence length and head dimension, while self-attention remains quadratic in sequence length.

\subsection{Static latency analysis.}
\label{app:latency}
Table~\ref{tab:latency} reports static latency based on sampled training and validation
iterator throughput during the second epoch for the Pixel1M and Shop examples.
T-RoPE has modest observed throughput overhead relative to baseline HSTU and is comparable to TO-RoPE.

\begin{table}[!ht]
\caption{Static throughput snapshots for baseline HSTU, TO-RoPE, and T-RoPE. Values are mean $\pm$ sample standard deviation in iterations per second.}
\label{tab:latency}
\centering
\small
\begin{tabular}{@{}llcc@{}}
\toprule
Dataset example & Model & Training it/s & Validation it/s \\
\midrule
Pixel1M & Baseline HSTU & $17.58 \pm 0.01$ & $58.59 \pm 0.87$ \\
Pixel1M & HSTU + TO-RoPE & $16.84 \pm 0.05$ & $57.33 \pm 0.74$ \\
Pixel1M & HSTU + T-RoPE & $16.95 \pm 0.01$ & $57.06 \pm 0.44$ \\
\midrule
Shop & Baseline HSTU & $37.19 \pm 0.20$ & $92.66 \pm 0.69$ \\
Shop & HSTU + TO-RoPE & $32.17 \pm 0.35$ & $82.62 \pm 0.92$ \\
Shop & HSTU + T-RoPE & $32.91 \pm 0.05$ & $86.05 \pm 0.71$ \\
\bottomrule
\end{tabular}
\end{table}



\section{Dataset Statistics}
\label{app:datasets}

Table~\ref{tab:datasets} reports the size of the five public benchmarks used in
Section~\ref{sec:public_experiments}.

\begin{table}[!ht]
\caption{Statistics of the five public benchmarks.}
\label{tab:datasets}
\centering
\small
\begin{tabular}{@{}lrrr@{}}
\toprule
Dataset & \#User & \#Item & \#Interaction \\
\midrule
ML-20M~\citep{harper2015movielens}   &     138,493 &  26,744 &   20,000,263 \\
Amazon Books~\citep{hou2024bridging} &     776,370 & 495,063 &    9,488,297 \\
\midrule
Pixel200K~\citep{cheng2024pixelrec}  &     200,000 &  96,282 &    3,965,656 \\
Pixel1M~\citep{cheng2024pixelrec}    &   1,001,822 & 100,541 &   19,886,579 \\
Pixel8M~\citep{cheng2024pixelrec}    &   8,886,078 & 407,082 &  158,488,682 \\
\bottomrule
\end{tabular}
\end{table}

\section{Hyperparameter Settings}
\label{app:hyperparams}

Table~\ref{tab:hyperparams} summarises the architecture and training hyperparameters.
All models use 2 Transformer blocks, 128 in-batch negatives, and a learning rate of $10^{-3}$.
The values shown are the defaults shared by ML-20M, Amazon Books, Pixel200K, Pixel1M, and
Pixel8M. The sequence-length exception is listed in the footnote.
Here $d$ is the embedding dimension, $H$ the number of attention heads, $A$ the attention dimension, and
RAB the relative attention bucket type (Pos = position only; T+P = time + position); ``---'' denotes an absent component.

\begin{table}[!ht]
\caption{Hyperparameter settings (defaults for all five datasets; see footnote for the sequence-length exception).}
\label{tab:hyperparams}
\centering
\small
\begin{tabular}{@{}lcccccc@{}}
\toprule
Method & $d$ & $H$ & $A$ & RAB & Bkts & RoPE \\
\midrule
SASRec         & 512 & 4 & ---  & ---  & ---  & ---    \\
TiSASRec       & 512 & 4 & ---  & ---  & ---  & ---    \\
HSTU           & 512 & 4 & 128  & Pos  & ---  & ---    \\
HSTU+Time RAB  & 512 & 4 & 128  & T+P  & 128  & ---    \\
HSTU+TO-RoPE   & 512 & 4 & 128  & T+P  & 128  & TO-RoPE \\
HSTU+T-RoPE    & 512 & 4 & 128  & T+P  & 128  & T-RoPE \\
\bottomrule
\end{tabular}
\smallskip

\footnotesize
$^\dagger$ Max sequence length $L=200$ for ML-20M; $L=50$ for all other datasets.
\end{table}

For public benchmark training, we use AdamW with learning rate $10^{-3}$, betas
$(0.9,0.98)$, weight decay $0.001$, batch size $128$, seed $42$, and dropout $0.2$
for all dropout settings. Models are trained for at least 10 and at most 100 epochs, with a
ReduceLROnPlateau scheduler on validation HR@10. The objective is SampledSoftmaxLoss
with 128 sampled negatives and temperature $0.05$. Histories are truncated to the
dataset maximum sequence length. All public training and evaluation runs use a single
NVIDIA H100.

\section{Additional Visualizations}
\label{app:figs}

The output attention of T-RoPE is visualized in Figure~\ref{fig:multiscale}, which shows that the day-, week-, month-, and
year-scale event blocks become increasingly grouped by their real time spacing as each
component is added, until full T-RoPE cleanly separates the blocks by elapsed time. 

\begin{figure}[h]
    \centering
    \includegraphics[width=\linewidth,trim={11 4 12 6},clip]{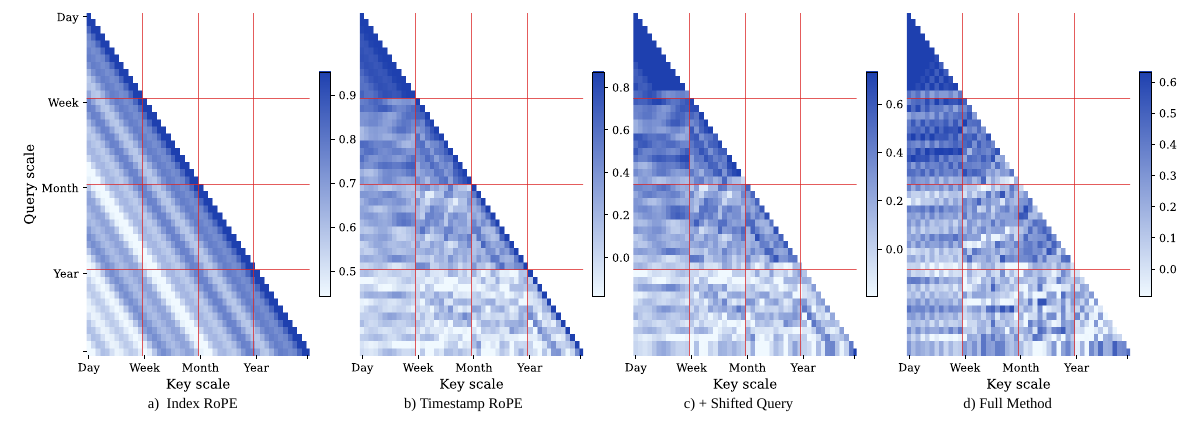}
    \caption{Attention geometry across T-RoPE design stages on a synthetic multiscale sequence. The sequence contains four event blocks whose internal spacing corresponds to day-, week-, month-, and year-scale behavior; red lines mark block boundaries. a) Index RoPE. b) Timestamp RoPE. c) Timestamp RoPE with shifted queries. d) Full T-RoPE with non-stationary keys and learned temporal coefficients. The full method separates the blocks by time geometry rather than only by sequence position.}
    \label{fig:multiscale}
\end{figure}

\end{document}